\documentclass[letterpaper]{article} 
\usepackage[preprint]{aaai2027}  
\usepackage{times}  
\usepackage{helvet}  
\usepackage{courier}  
\usepackage[hyphens]{url}  
\usepackage{graphicx} 
\usepackage{natbib}  
\usepackage{caption} 
\usepackage{amsmath}
\usepackage{amssymb}
\usepackage{amsthm}
\usepackage{mathtools}
\mathtoolsset{showonlyrefs}
\usepackage{algorithm}
\usepackage{algorithmic}
\usepackage{booktabs}

\usepackage{newfloat}
\usepackage{listings}
\DeclareCaptionStyle{ruled}{labelfont=normalfont,labelsep=colon,strut=off} 
\floatstyle{ruled}
\newfloat{listing}{tb}{lst}{}
\floatname{listing}{Listing}
\usepackage{comment}

\title{Learning Optimal Dynamic Matching via Graph Neural Networks}
\author{
    Genta Okada\textsuperscript{\rm 1},
    Shunya Noda\textsuperscript{\rm 1},
    Junpei Komiyama\textsuperscript{\rm 2},
    Akira Matsushita\textsuperscript{\rm 1}
}
\affiliations{
    \textsuperscript{\rm 1}The University of Tokyo\\
    \textsuperscript{\rm 2}Mohamed bin Zayed University of Artificial Intelligence
}

\newtheorem{theorem}{Theorem}
\newtheorem{lemma}[theorem]{Lemma}

\newcommand{\R}{\mathbb{R}}
\newcommand{\Rp}{\mathbb{R}_{\ge 0}}
\newcommand{\N}{\mathbb{N}}
\newcommand{\E}{\mathbb{E}}
\newcommand{\Prob}{\mathbb{P}}
\newcommand{\X}{\mathcal{X}}
\newcommand{\G}{\mathcal{G}}
\newcommand{\M}{\mathcal{M}}
\newcommand{\res}{\mathbin{\ominus}}
\newcommand{\dd}{\mathrm{d}}
\newcommand{\ent}{\mathrm{ent}}
\newcommand{\exitmark}{\dagger}
\newcommand{\Exp}{\mathrm{Exp}}

\begin{document}

\maketitle

\begin{abstract}
Dynamic matching markets require decisions about whom to match and when: matching now yields value but removes participants who may create better future opportunities. We develop a value-based reinforcement-learning framework for this problem on finite, evolving weighted graphs. We study an infinite-horizon continuous-time model with stochastic arrivals, node-type transitions, edge realizations, and exogenous exits. We prove an event-time reduction: without loss of optimality, the planner acts immediately after each exogenous event and then waits for the next one. We further show that the optimal edge-wise $Q$-function is characterized by a single continuation-value function on post-decision residual graphs, reducing the learned object from state-action values to graph values. Exact action selection still requires combinatorial matching optimization; we approximate the value with a graph neural network, train it by temporal-difference learning, and use it in a forward-greedy matching heuristic. In a binary-type benchmark, the learned policy substantially outperforms immediate and threshold-greedy rules by preserving common nodes for rare arrivals of valuable matches while forming lower-value matches only in thick pools. In a kidney paired donation benchmark, it performs similarly to immediate greedy when exits are unpredictable, recovers the logic of patient matching when warnings are reliable, and outperforms the better of Immediate Greedy and Patient Greedy across intermediate warning probabilities. These results show that residual-graph value learning yields state-dependent dynamic matching policies that adapt to realized connectivity and exit information.
\end{abstract}


\section{Introduction}

Matching problems are central to market design and have been studied
in economics, operations research, and computer science.
Matching theory has provided some of the most successful and practically deployed tools for
allocating agents, objects, and opportunities, ranging from school choice and
labor markets to organ exchange. Much of the classical literature studies
static problems in which the set of nodes (representing agents or objects) and
edges (representing feasible or valuable matches) are fixed. Many important applications, however, are
inherently dynamic, meaning that the composition of the pool evolves over time as new nodes arrive and existing nodes leave.
In such environments, optimal matching requires more than simply choosing
which currently feasible matches to execute.
A planner must also decide \emph{when} to execute them, trading off immediate match
value against the option value of preserving market thickness
and the risk that current participants leave the market.

To see the tradeoff, consider kidney paired donation, where incompatible patient--donor pairs are connected when they can form a two-way exchange. Executing an available exchange immediately reduces the risk that a patient leaves unmatched. Yet it also removes both pairs from the pool and may foreclose valuable exchanges with future arrivals, especially for hard-to-match patients. Waiting can therefore improve match quality by preserving market thickness, but it exposes current patients to delay and departure risk. The same tension arises in ride sharing: assigning a driver immediately reduces rider waiting, whereas a delay may yield a closer driver and avoid an inefficient long-distance pickup \citep{10.1287/mnsc.2022.00096}. Analogous tradeoffs appear in online dating and other platform markets. Optimal policies must balance immediate match value against the continuation value from future opportunities.


Dynamic matching is computationally challenging: continuation values depend
on the combinatorial composition and realized structure of the pool, so even
with known arrival and edge-value distributions, deriving an optimal
state-dependent policy remains difficult.
Existing work clarifies the timing tradeoff but does not directly solve the
state-dependent problem here. Classical online matching typically requires
immediate and irrevocable decisions upon arrival, leaving no endogenous choice
of how long to preserve the pool \citep[see, e.g.,][]{huang2024online}. Dynamic-matching
studies analyze waiting, thickness, and abandonment, but their guarantees
exploit random-graph or asymptotic structure, or type-level relaxations, rather
than the realized graph
\citep{anderson2017efficient,akbarpour2020thickness,ashlagi2019matching,DBLP:journals/ior/AouadS22}.
Recent approximate dynamic programming and GNN methods broaden computational
scope, yet rely on structured value approximations or online bipartite arrival
models \citep{you2024approximate,DBLP:conf/icml/HayderiSVW24}. What remains
difficult is learning how each finite residual graph---including realized
connectivity, evolving node states, and exit information---shapes continuation
value.


This gap motivates a flexible, simulation-based approach that learns
continuation values directly from realized graph states.
This paper proposes a value-based computational framework for dynamic matching
using deep reinforcement learning with \emph{graph neural networks} (GNNs). We study an
infinite-horizon continuous-time model with stochastic entry, type changes,
edge realizations, and exogenous exit. We derive an event-time reduction:
after each exogenous event, it is without loss to choose a matching
immediately and then wait until the next event. The event-time reduction removes the need to optimize over arbitrary intervention times and yields a Bellman system in which each decision trades off the immediate payoff from matching against the continuation value of the post-decision residual graph.

We approximate this residual-graph value function by a GNN
and train it with a temporal-difference procedure \citep{sutton1988learning},
a simulation-based approximate dynamic-programming method. This design is motivated by
recent progress in automated mechanism design, where neural networks are used
to represent and optimize mechanisms over large design spaces \citep{dutting2019optimal}.
GNNs provide a scalable class of message-passing architectures for graph-structured inputs and are especially well suited to dynamic matching because unmatched
participants and feasible matches form a graph, and the value of a pool should
be invariant to arbitrary labels assigned to participants. Once learned, the
value function guides a forward-greedy approximation to the matching-level
Bellman maximization, avoiding exhaustive search over feasible matchings.

We evaluate the method in two environments. The first is a stylized
binary-type benchmark designed to isolate the timing tradeoff in dynamic
matching: although there are only two node types, good performance requires
preserving some common, slow-exiting nodes as future partners for rare and
perishable arrivals. In this benchmark, the learned policy discovers a state-dependent timing rule that selectively preserves common nodes for future rare arrivals while executing lower-value matches when the pool is sufficiently thick. It substantially outperforms both immediate greedy and threshold-greedy benchmarks, showing that the framework can discover sophisticated matching policies beyond fixed, hand-designed heuristics.

The second environment is a kidney paired donation (KPD) benchmark with heterogeneous patient--donor types, stochastic compatibility, edge values, and exogenous exit. In this environment, the learned policy uses both the informativeness of exit signals and the realized structure of the pool to determine when and whom to match. When exits are unpredictable, it behaves similarly to immediate greedy, because waiting provides little protection against sudden departures. When a critical-status flag reliably warns of exit, it recovers the central logic of patient matching \citep{akbarpour2020thickness}, preserving noncritical nodes to maintain market thickness while prioritizing nodes at imminent risk. The learned policy is not restricted to either heuristic, however: it can match hard-to-match nodes before they become critical and select among feasible critical-node matches according to their effects on the residual pool. In intermediate environments, where exit warnings arrive only probabilistically, it adapts the degree of waiting to the available information and outperforms both baselines. These results show that the framework can discover state-dependent matching policies tailored to the exit-information structure, rather than merely selecting the better of two fixed rules. More broadly, it provides a computational laboratory for studying how allocation, timing, and information interact in dynamic matching markets.


Our work contributes to online and dynamic matching by learning
continuation values for finite, realized weighted graphs. Recent studies
have applied GNNs and related neural architectures to non-static matching
problems
\citep{DBLP:conf/acml/BegnardiBJ023,DBLP:conf/icml/HayderiSVW24,
Liu2025linksage,DBLP:journals/tits/HuFL25,Xing2025RidePoolingSTGNN}.
Recent work also develops optimization-based and data-driven timing
policies, including supply-aware simulation and learned switching between
Greedy and Patient rules
\citep{DBLP:journals/ior/AouadS22,chen2025feature,eom2025batching,
liu2026datadriven,zhou2026hybrid}.
Our distinction is in the learned object: we learn a graph-level
continuation value for each realized post-decision residual graph. With an
appropriate graph encoder, decisions can adapt directly to the realized pool,
connectivity, node and edge attributes, type changes, and exit information.

\section{Problem}

We study an infinite-horizon dynamic matching problem in continuous time
$t\in \R_{\ge 0}$. A planner discounts future payoffs at rate
$\delta>0$. The \emph{environment} is defined as
$\mathcal{E}=(\X,\rho,F,\lambda,\mu,q,c,\delta)$. The set of node types
$\X$ is finite. 
We use the term ``exogenous event'' to denote changes in the state (i.e.,
entry and exit of nodes, and type changes of nodes). These events are
stochastic and occur at random times regardless of the planner's matching
decisions.
New entrants arrive at rate $\lambda>0$, and their
types are independently drawn from $\rho\in \Delta(\X)$. For each pair of types
$(x,y)\in \X^2$, $F(\cdot\mid x,y)$ is a probability distribution on
$\Rp \cup \{\bot\}$, where $\bot$ denotes infeasibility; a realization
$u\in \Rp$ denotes a feasible edge of weight $u$. We assume
$F(\cdot\mid x,y)=F(\cdot\mid y,x)$ for all $(x,y)\in \X^2$ and $\int_{\Rp} u F(\dd u\mid x, y) < \infty$. Each
unmatched node of type $x\in \X$ has a Poisson clock with rate
$\mu(x)>0$ and transition kernel
$q(\cdot\mid x)\in \Delta(\X\cup \{\exitmark\})$ over new types and
exogenous exit. If a node of type $x$ exits exogenously, the planner
incurs penalty $c(x)\ge 0$. Throughout, all Poisson processes are
independent, and all newly created edge values are conditionally
independent given the current node types.

This transition kernel enables us to model exit information explicitly.
Specifically, an observable critical-warning signal can be included as part of
the node type. A noncritical type may either transition to the corresponding
critical type or exit directly when its clock rings, while a critical type may
exit at a different hazard rate. Varying the probability of the transition to
the critical type therefore captures environments ranging from uninformed exit,
in which no warning is observed before departure, to informed exit, in which a
warning reliably arrives before departure.
Motivated by this interpretation, we do not redraw edge weights
after type transitions in any of our applications.


A \emph{state} is a finite typed weighted graph $G=(N,E,x,w)\in \G$, where $N\subset \N$ is a finite node set, $E\subseteq \binom{N}{2}$ is the set of feasible edges, $x:N\to \X$ assigns a type to each node, and $w:E\to \Rp$ assigns a realized weight to each feasible edge.

Note that the Markov state cannot be summarized by the profile of node types alone. Even when two pools have the same multiset of node types, they may have different realized feasible edges and different realized edge weights, and therefore different continuation values. For example, in KPD, a two-way exchange requires both cross-donations to pass pair-specific virtual crossmatches, so coarse patient--donor types alone do not determine whether the exchange is feasible. Hence the relevant state variable is the realized weighted graph itself.

Given a state $G$, the exogenous dynamics act as follows. New nodes enter according to a Poisson process with rate $\lambda$. When a new node $i^*\notin N$ enters, its type $y$ is drawn from $\rho$. For each existing node $j\in N$, a fresh edge realization $\tilde w_{i^* j}\sim F(\cdot\mid y,x_j)$ is drawn independently. If $\tilde w_{i^* j}=\bot$, no edge is
created; otherwise $\{i^*,j\}$ is added to the feasible-edge set with
weight $\tilde w_{i^* j}$.

Each remaining node $i\in N$ has an independent Poisson clock with rate $\mu(x_i)$. When the clock of node $i$ rings, $z\sim q(\cdot\mid x_i)$ is drawn. If $z=\exitmark$, then node $i$ exits exogenously, the planner incurs a penalty $c(x_i)$, and node $i$ together with all incident edges is removed from the graph. If $z=y\in \X$, then node $i$ changes type from $x_i$ to $y$. 
All existing feasible
edges and realized weights, including those incident to $i$, are left
unchanged.

For a state $G=(N,E,x,w)$, a feasible matching is
\begin{equation}
\M(G)
\coloneqq
\{m\subseteq E:\ 
e\cap e'=\emptyset \: \forall \text{distinct } e,e'\in m\}.
\end{equation}
The empty matching $\emptyset$ is allowed.

If the planner chooses a matching $m\in \M(G)$, it receives immediate
payoff
\begin{equation}
W(G,m)\coloneqq\sum_{e\in m} w_e.
\end{equation}
Let $\nu(m)$ denote the set of matched nodes in $m$. The residual graph after matching $m$ is $G\res m$, defined as the induced subgraph obtained by deleting all nodes in $\nu(m)$
together with all incident edges.

An admissible policy is a non-anticipative, possibly randomized rule that chooses feasible matchings at stopping times adapted to the filtration generated by past exogenous events and past actions. Any private randomization used by the planner is assumed independent of the primitive Poisson clocks and edge draws.

For an admissible policy $\pi$, let $(\sigma_k,M_k)_{k\ge 0}$ denote the
sequence of matching interventions it generates, where $\sigma_k$ is a
stopping time, $G_{\sigma_k}$ is the graph observed immediately before
that intervention, and $M_k\in \M(G_{\sigma_k})$.
Let $(\eta_\ell,X_\ell)_{\ell\ge 1}$ denote the sequence of exogenous exit
events, where $X_\ell$ is the type of the exiting node immediately before
exit at time $\eta_\ell$. The expected discounted payoff from the initial graph
$G$ is
\begin{equation}
    J^\pi(G) = 
    \E^\pi_G
\biggl[\sum_{k\ge 0} e^{-\delta\sigma_k} W(G_{\sigma_k},M_k)-\sum_{\ell\ge 1} e^{-\delta\eta_\ell} c(X_\ell)
\biggr].\label{eq:continuous_time_objective}
\end{equation}
Let $U(G)$ denote the optimal value of the full continuous-time problem
from a pre-decision graph $G$,
\begin{equation}
U(G)\coloneqq\sup_{\pi} J^\pi(G).
\end{equation}

\section{Method}
\label{sec:method}


\subsection{Event-time Reduction}\label{subsec:reduction}

We first present a few analytical results to simplify our learning problem.
Let $0=\tau_0<\tau_1<\tau_2<\cdots$ denote the sequence of exogenous event
times (i.e.,
entry and exit of nodes, and type changes of nodes). For an event-time policy, let $G_n$ be the pre-decision graph
observed immediately after the $n$-th exogenous event, let
$m_n\in \M(G_n)$ be the matching chosen at time $\tau_n$, and let
$R_n\coloneqq G_n\res m_n$ be the (post-decision) \emph{residual graph}.

For a residual graph $R=(N,E,x,w)$, define the \emph{total hazard} by $\Lambda(R)\coloneqq\lambda+\sum_{i\in N}\mu(x_i)$.
The waiting time to the next exogenous event, conditional on $R$, satisfies
\begin{equation}
\Delta_n\coloneqq\tau_{n+1}-\tau_n \sim \Exp(\Lambda(R)),
\end{equation}
where $\Exp(\Lambda(R))$ denotes the exponential distribution with rate $\Lambda(R)$.
The associated event mark $\xi_n$, which identifies the next exogenous event, takes values in
\begin{equation}
\Xi(R)\coloneqq\{\ent\}\cup \{(i,z): i\in N(R),\ z\in \X\cup\{\exitmark\}\}.
\end{equation}
Under competing Poisson hazards, the probability of the entry event is
\begin{equation}
\Prob(\xi_n=\ent\mid R_n=R)=\frac{\lambda}{\Lambda(R)},
\end{equation}
and for each node $i\in N(R)$ and mark $z\in \X\cup\{\exitmark\}$, the probability of the event that node $i$'s clock rings and yields mark $z$ is
\begin{equation}
\Prob(\xi_n=(i,z)\mid R_n=R)
=
\frac{\mu(x_i)\,q(z\mid x_i)}{\Lambda(R)}.
\end{equation}
Moreover, $\Delta_n$ and $\xi_n$ are independent conditional on $R_n$.

Let $T(R,\xi_n)$ denote the random next pre-decision graph obtained from
$R$ after applying event mark $\xi_n$ together with the auxiliary edge
draws, if any, induced by entry. Define the payoff of an exit event as
\begin{equation}
r^{\mathrm{ex}}(R,\xi_n)
\coloneqq
\begin{cases}
-c(x_i), & \text{if }\xi_n=(i,\exitmark),\\
0, & \text{otherwise.}
\end{cases}
\end{equation}

For such an event-time policy, the value $J^\pi(G_0)$ can be written as
\begin{align}
&J^\pi(G_0)\\
& =  
\E^\pi_{G_0}
\biggl[
\sum_{n=0}^\infty
e^{-\delta \tau_n}
\Bigl(W(G_n,m_n)+
e^{-\delta \Delta_n}r^{\mathrm{ex}}(R_n,\xi_n)
\Bigr)
\biggr].
\end{align}
Let $V(R)$ denote the post-decision value obtained when the planner takes
no further matching action from graph $R$ until the next exogenous event.
Then
\begin{equation}
V(R)
=
\E\!\left[
e^{-\delta \Delta_R}
\bigl(
r^{\mathrm{ex}}(R,\xi_R)+U(T(R,\xi_R))
\bigr)
\ \middle|\ R
\right],
\label{eq:bellman_post_raw}
\end{equation}
where $\Delta_R\sim \Exp(\Lambda(R))$ and $\xi_R$ is the corresponding
event mark. Since $\Delta_R$ and $\xi_R$ are independent conditional on
$R$,
\begin{align}
V(R)
&=
\Gamma(R)\,
\E\!\left[
r^{\mathrm{ex}}(R,\xi_R)+U(T(R,\xi_R))
\ \middle|\ R
\right],
\label{eq:bellman_post}
\\
\Gamma(R)
&\coloneqq
\frac{\Lambda(R)}{\delta+\Lambda(R)}.
\end{align}

\begin{lemma}[No profitable interior delay]
\label{lem:no_delay}
Fix a residual graph $R\in\G$. Any plan that waits until a positive
interior time to execute a planned matching, unless an exogenous event
arrives first, has continuation value at most
\begin{equation}
\max\left\{
V(R),\,
\max_{m\in \M(R)\setminus\{\emptyset\}}
\bigl(
W(R,m)+U(R\res m)
\bigr)
\right\}.
\label{eq:no_delay_bound}
\end{equation}
\end{lemma}


Lemma~\ref{lem:no_delay} shows that the first planned matching
intervention at a strictly positive interior time can be weakly improved
by either waiting for the next exogenous event or executing that matching
immediately and continuing optimally.
In other words, making a matching decision between exogenous events does not improve the value.
We now apply this observation
recursively. Since every nonempty matching removes at least two nodes, the
recursion terminates after finitely many immediate interventions. This
yields the following event-time reduction.

\begin{theorem}[Event-time reduction]
\label{thm:event_time}
For every pre-decision graph $G\in \G$,
\begin{equation}
U(G)
=
\max_{m\in \M(G)}
\left\{
W(G,m)+V(G\res m)
\right\}.
\label{eq:bellman_pre}
\end{equation}
\end{theorem}
Note that $\max_{m\in \M(G)}$ in \eqref{eq:bellman_pre} includes the empty matching $m=\emptyset$.
Theorem~\ref{thm:event_time} states that the optimal policy can be represented as a sequence of event-time decisions: after each exogenous event, the planner chooses one matching immediately and then waits until the next exogenous event. The planner never needs to plan a delayed match that executes at an interior time between exogenous events.
Since $\M(G)$ is finite for each finite graph, an optimal event-time policy exists.

Together, \eqref{eq:bellman_post} and \eqref{eq:bellman_pre} form the
Bellman system for the event-time Markov decision problem.

\subsection{Reduction to Residual-graph Values}\label{subsec:reduction_q_to_v}

The post-decision formulation above implies a useful equivalent edge-wise
representation. Within a single event epoch, starting from a graph $G$, the
planner may repeatedly choose either a feasible edge
$e\in E(G)$ or the stop action $\perp$. Choosing edge $e$ yields
immediate payoff $w_e$ and transitions instantaneously to the residual
graph $G\res \{e\}$. Choosing $\perp$ yields no immediate matching payoff
and hands the process over to the exogenous dynamics.

Let $Q(G,e)$ for $e\in E(G)$ and $Q(G,\perp)$ denote the optimal
action values in this edge-wise representation.

\begin{theorem}[Reduction to residual-graph values]
\label{thm:q_to_v}
For every graph $G\in \G$, the following identities hold:
\begin{align}
Q(G,e)
&=
w_e+U(G\res \{e\}),
\qquad\text{ for } e\in E(G),
\label{eq:q_edge}
\\
Q(G,\perp)
&=
V(G).
\label{eq:q_stop}
\end{align}
Moreover,
\begin{align}
U(G)
&=
\max\Bigl(
\{Q(G,\perp)\}\cup\{Q(G,e):e\in E(G)\}
\Bigr)
\nonumber\\
&=
\max_{m\in \M(G)}
\left\{
W(G,m)+V(G\res m)
\right\}.
\label{eq:q_to_v_reduction}
\end{align}
\end{theorem}

Thus the optimal edge-wise $Q$-function is characterized by the single
post-decision value function $V$, but recovering it requires evaluating $U$
through the matching maximization in \eqref{eq:q_to_v_reduction}. If $R_n$ is the realized residual graph at event
epoch $n$, $\xi_n$ is the realized next event mark, and
$G_{n+1}=T(R_n,\xi_n)$ is the realized next pre-decision graph, then a
natural TD update for the exact Bellman system is
\begin{align}
&V_{n+1}(R_n)-V_n(R_n)\\
&= \alpha_n
\biggl[
\Gamma(R_n)
\Bigl(
r^{\mathrm{ex}}(R_n,\xi_n)+U_n(G_{n+1})
\Bigr)
-
V_n(R_n)
\biggr],
\end{align}
where $\alpha_n\in (0,1)$ is a step size and
\begin{equation}
U_n(G)
=
\max_{m\in \M(G)}
\left\{
W(G,m)+V_n(G\res m)
\right\}.
\end{equation}
This formulation avoids parameterizing the full state-action space of
edge-wise $Q$-values, but retains a combinatorial maximization, which we
approximate below.

\subsection{Value Approximation}\label{sec:value_approx}

Even after applying these reductions, the learning problem
remains very large. The possible combinations of existing node types grow
combinatorially with the size of the pool. Moreover, in many
applications, realized edge values are naturally continuous. 
For example, in kidney exchange, an edge value can summarize
multiple clinical determinants of expected transplant outcomes. A tabular
representation of $V$ is therefore infeasible. 

This motivates approximating the value function with a GNN. The role of the
neural network is not to impose a structural interpretation on hidden units.
Rather, it provides a scalable parametric class for representing continuation
values on variable-size graphs. GNNs are natural in this setting because they
share parameters across nodes and use the realized edge structure in message
passing. Before pooling, their node representations are equivariant to node
relabeling; after pooling, they produce permutation-invariant graph-level
representations. Thus, they can use the realized graph structure directly while
using far fewer parameters than an unrestricted representation of graph values.

We approximate the post-decision value function by a GNN
$V_\theta:\G\to\R$ defined on residual graphs. Each residual graph
$R=(N,E,x,w)$ is encoded by node features and a directed edge list containing
both orientations of every feasible undirected edge. The directed edges may
additionally carry edge features, including the realized weight $w_e$. The
node features encode type information relevant for type transitions, exits,
and compatibility with future arrivals. When realized edge weights are
informative for continuation value, an edge-aware message-passing operator
can condition on these features. The residual-value reduction and training
procedure do not otherwise depend on this encoding choice.

For the reported experiments, the value network consists of three
graph-convolution layers, global additive pooling, and a two-layer readout. Let $d_x$ denote the node-feature
dimension and let $h$ denote the hidden dimension. The first graph-convolution
layer maps $d_x$ input features to $h$ hidden units, and the next two layers
map $h$ to $h$. Each graph-convolution layer is followed by a ReLU
nonlinearity. If $z_i$ denotes the final embedding of node $i$, the graph
representation is $z_R = \sum_{i\in N(R)} z_i$.
The additive pooling operation makes the graph-level representation invariant
to node relabeling and allows the same network to be evaluated on residual
graphs of different sizes. The pooled embedding is passed through a fully
connected hidden layer with ReLU activation and then through a scalar output
layer. The implemented value function has the form
\begin{equation}
V_\theta(R)
=
a^\top \sigma(Az_R+b)+b_0+\kappa,
\end{equation}
where $\sigma$ is ReLU and $\kappa$ is a learned scalar offset. Throughout the experiments, we set $h=64$.

\subsection{Matching Policy}\label{subsec:matching_policy}

Given $V_\theta$, exact plug-in action selection maximizes
$F_\theta^G(m)\coloneqq W(G,m)+V_\theta(G\res m)$ over $m\in\M(G)$.
This remains combinatorial because the nonlinear GNN term need not decompose
into edge weights. We therefore use forward-greedy ascent. For a partial
matching $m$ and $H=G\res m$, adding $e\in E(H)$ changes $F_\theta^G$ by
$w_e+V_\theta(H\res\{e\})-V_\theta(H)$. Hence the procedure adds an edge
maximizing $w_e+V_\theta(H\res\{e\})$ if this score exceeds $V_\theta(H)$,
deletes its endpoints, and repeats; otherwise it stops. Let $m_\theta(G)$
denote its output. Each accepted edge increases the same plug-in Bellman
objective, so the output has no profitable feasible single-edge addition under
$F_\theta^G$, but it need not be globally optimal because earlier choices are
not revisited and a profitable combination may have no profitable first edge.

Under the fixed evaluation policy and after initialization, each event-time
pre-decision graph is obtained by a single-node arrival, exit, or type change
from the residual graph at which the previous search stopped. Thus forward
local reoptimization fits the one-at-a-time Poisson event structure, although
this provides no optimality guarantee. All reported simulations use this
heuristic and assess its performance against benchmark policies.

\subsection{Training}\label{subsec:training}

Training uses the residual-graph decomposition in Theorem~\ref{thm:q_to_v},
with its exact matching maximization replaced by the forward-greedy procedure.
 The agent maintains an online value network
$V_\theta$ and a target network $V_{\bar\theta}$, where $\bar{\theta}$ is a delayed copy of the online parameters $\theta$.
The online network is used for action selection, while the target network is
used only for evaluation.
A training episode first
runs the simulator for a warm-up period and then records one transition per
exogenous event. Each experience has the form $(R_n,\gamma_n,r^{\mathrm{ex}}_n,G_{n+1})$, where $R_n$ is the post-decision residual graph, $\gamma_n=\Gamma(R_n)$ is the state-dependent effective discount factor, $r^{\mathrm{ex}}_n$ is the realized exogenous-event reward, and $G_{n+1}$ is the next pre-decision graph.

For each sampled next pre-decision graph $G$, the online network selects
$m_\theta(G)$ using the procedure above. The target network then evaluates the
plug-in value of the selected matching:
\begin{equation}
    \widehat U(G)
    =
    W(G,m_\theta(G))
    +
    V_{\bar\theta}\bigl(G\res m_\theta(G)\bigr).
\end{equation}
This approximate pre-decision value gives the TD target
\begin{equation}
    y_n
    =
    \gamma_n
    \left(
    r^{\mathrm{ex}}_n+\widehat U(G_{n+1})
    \right).
\end{equation}
The online network is trained to predict this target at the sampled
post-decision residual graph.

Experiences are stored in a prioritized replay buffer
\citep{schaul2016prioritized}. A minibatch of size $32$ is sampled using
proportional prioritization with exponent $0.6$. 
The importance-sampling exponent is annealed from 0.4 toward 1.0.
Priorities are updated after
each gradient step using the absolute TD error, clipped at one and shifted
by $0.01$. The network minimizes the importance-weighted squared TD error
\[
\mathcal L(\theta)
=
\frac{1}{B}\sum_{i=1}^B
\omega_i
\left(
V_\theta(R_i)-y_i
\right)^2,
\]
where $\omega_i$ is the importance-sampling weight of sample $i$. The network
parameters are updated by stochastic gradient steps using the Adam optimizer
\citep{kingma2015adam}.

Exploration is performed by an $\varepsilon$-greedy version of the same
edge-wise matching procedure. With probability $\varepsilon$, the agent
chooses the stop action with a benchmark-specific probability and otherwise
samples uniformly from the currently feasible edges. With probability
$1-\varepsilon$, it takes the greedy action under $V_\theta$. The exploration
rate is multiplied by a decay factor after each gradient update and is
bounded below by a minimum value.

Our training procedure shares several components with DQN \citep{DBLP:journals/nature/MnihKSRVBGRFOPB15}, including an evolving $\varepsilon$-greedy behavior policy, experience replay, and a lagged target network. Unlike DQN, however, we learn a post-decision residual-graph value function rather than a state-action $Q$-function.

\begin{table*}[t]
\centering
\caption{Normalized performance relative to Immediate Random and Omniscient Bound}
\label{tab:performance_gap_closed}
\begin{tabular*}{0.8\textwidth}{lcccccc}
\toprule
Environment
& \shortstack{Immediate\\Random}
& \shortstack{Immediate\\Greedy}
& \shortstack{Immediate\\Threshold Greedy}
& \shortstack{Patient\\Greedy}
& GNN
& \shortstack{Tabular\\Optimal} \\
\midrule
Binary           & 0.00 & 0.01 & 0.27 & --   & 0.61 & 0.61 \\
KPD (uninformed) & 0.00 & 0.12 & 0.12 & --   & 0.12 & --   \\
KPD (informed)   & 0.00 & 0.12 & 0.12 & 0.40 & 0.44 & --   \\
\bottomrule
\end{tabular*}

\vspace{0.5em}
\begin{minipage}{0.95\textwidth}
\footnotesize
\emph{Note.}
Within each environment, rewards are normalized so Immediate Random equals 0
and Omniscient Bound equals 1. The latter observes all future arrivals, edge
realizations, and exits and solves the offline matching problem.
\end{minipage}
\end{table*}

\section{Binary Type Benchmark}
\label{subsec:binary_benchmark}

\subsection{Setting}

We first study a simple stylized benchmark with two node types and deterministic edge weights. 
The discount rate is $\delta = 0.002$. There are two node types,
$\X=\{h,l\}$, where $h$ denotes a rare, high-value and more perishable
type, while $l$ denotes a common, lower-value and more persistent type.
Nodes arrive at rate $\lambda=2.0$, with $\rho(l)=0.7$ and $\rho(h)=0.3$.
Unmatched nodes exit exogenously with type-specific hazards $\mu(l)=0.1$ and $\mu(h)=0.5$, with zero exit penalty, and no type transitions. Edges are generated independently conditional on node types. The edge
existence probabilities are $\phi_{hh} = 0.05$, $\phi_{hl} = 0.95$, and $\phi_{ll} = 0.8$. Edge weights are specified by node types: $w_{hh} = 5$, $w_{hl} = 5$, and $w_{ll} = 1$.

\subsection{Results}

Table~\ref{tab:performance_gap_closed} reports normalized discounted
rewards. ``Immediate Random'' selects a uniformly random feasible edge
whenever one exists, while ``Immediate Greedy'' immediately executes a
maximum-weight matching. ``Immediate Threshold Greedy (TG)'' greedily
matches edges above an optimally chosen threshold; under homogeneous exit
rates, it is a 2-approximation and optimal among greedy policies
\citep{DBLP:journals/ior/AouadS22,arnosti2025greedydynamicmatching}.
``GNN'' is the forward-greedy, value-guided policy above. For the binary
benchmark, exact tabular value iteration gives ``Tabular Optimal''; it is
infeasible for KPD.


\begin{figure}[t]
    \centering
    \includegraphics[width=\linewidth]{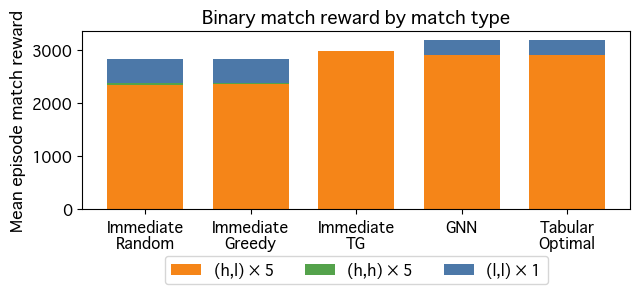}
    \caption{Reward composition by match type in the binary benchmark}
    \label{fig:binary_match_composition}
\end{figure}

The central dynamic tradeoff is how aggressively to use $l$-nodes. Type-$h$ nodes are rare and relatively perishable, while $h$--$l$ matches are both
likely and valuable. Accordingly, the planner benefits from preserving $l$-nodes as
potential partners for future $h$-nodes. Matching every available $l$--$l$
edge too quickly sacrifices these future opportunities. Preservation should
not be absolute, however: once sufficiently many $l$-nodes have accumulated
in the pool, the marginal value of an additional reserve is small, and
executing some low-value $l$--$l$ matches becomes desirable. A good policy
must therefore condition its willingness to form $l$--$l$ matches on the
current thickness of the pool.

Figure~\ref{fig:binary_match_composition} shows that the GNN policy achieves this
balance more effectively than Greedy or Threshold Greedy. Greedy uses
$l$-nodes too readily, whereas Threshold Greedy forms exclusively
$h$--$l$ matches, consistent with preserving $l$-nodes even when the pool is
already thick enough to make some $l$--$l$ matching worthwhile. GNN balances these two: most of its realized matches are the valuable
$h$--$l$ matches, but it also matches $l$--$l$ when the market is sufficiently thick.
Consequently, relative to Threshold Greedy, GNN obtains a similar reward
contribution from $h$--$l$ matches while also extracting substantial value
from $l$--$l$ matches. This state-dependent use of $l$-nodes is reflected
in its substantially higher normalized performance in
Table~\ref{tab:performance_gap_closed}, which is virtually identical to
that of Tabular Optimal. Thus, despite relying on function
approximation and forward-greedy action selection, GNN achieves
near-optimal performance in this benchmark.

Overall, the binary-type benchmark shows that the proposed method learns
more than a myopic preference for high current edge weights. Even in this
minimal two-type environment, the learned value function guides a policy
that manages market thickness in a dynamically meaningful way.

\section{Kidney Paired Donation}
\label{subsec:kidney_benchmark}

\subsection{Setting}

The kidney paired donation (KPD) benchmark is meant to create a simple but clinically interpretable dynamic matching problem. Each node represents
an incompatible patient--donor pair, and feasible edges represent possible
two-way exchanges. Node types retain the main coarse sources of
heterogeneity in KPD: patient ABO type, sensitization, donor ABO type,
donor age, and donor sex. 
Feasible edges are realized through ABO compatibility and virtual
crossmatch, while an edge-specific HLA draw contributes to their value.

The type distribution and rewards are calibrated from coarse empirical and
clinical regularities. Patient ABO and sensitization frequencies follow the
standard kidney-exchange simulation calibration of
\citet{roth2007efficient}, with high sensitization motivated by observed
waiting-time heterogeneity in KPD pools \citep{holscher2018matchrates}.
Donor age and sex are chosen to match broad living-donor patterns
\citep{lentine2023adrkidney}. The KPD-entry filter reflects the clinical
fact that exchange pools are populated by pairs with blood-type or crossmatch incompatibility \citep{roth2007efficient,leeser2012living}; ABO
compatibility follows the standard transplant compatibility chart
\citep{unos_abo_chart}. 
Edge values use the KLY (\emph{kidney life-years}) regression coefficients of
\citet{milner2016hla}. Donor age and sex are encoded in the endpoint node
types, whereas HLA match quality is drawn at the edge level.
KLY is an estimate of the kidney life-years expected from a transplant; thus it provides a clinically interpretable proxy for the value of a feasible
exchange.

For the KPD value network, an architecture ablation found that edge-weight
features left evaluation reward unchanged but reduced training stability, so
we report the encoding based on node features and feasible-edge structure.
This choice affects only continuation-value approximation: realized weights
still enter immediate payoffs and all forward-greedy edge scores exactly.
Because donor age and sex remain available through endpoint features, the only
incremental information in edge weights is edge-specific HLA. This HLA
information did not improve performance, plausibly because any feasible
exchange is highly valuable relative to HLA variation among feasible edges,
so residual-pool value depends mainly on connectivity and exit information.

The simulation parameters are chosen so that residual graphs contain
enough simultaneous feasible edges for policy choices to matter. In
particular, the arrival rate and crossmatch-pass probabilities are set to
produce a dense enough pool, while exit hazards keep the problem dynamic.
We evaluate both an uninformed-exit KPD environment and an informed-exit
variant in which a \emph{critical flag} gives the planner advance warning that a
node is likely to leave soon. The complete parameterization is given in
Appendix~\ref{app:kpd_details}.

\subsection{Results}

The discounted rewards achieved are presented in Table~\ref{tab:performance_gap_closed}. In the uninformed-exit KPD benchmark, Immediate Greedy, Immediate Threshold Greedy, and GNN achieve similar performance. Within this benchmark, it is difficult for a planner to control market thickness, and the learned continuation value offers little gain over repeatedly implementing the highest-value
currently available exchanges when exit risk is unobservable.

In the informed-exit KPD benchmark, we add Patient Greedy as a baseline. It
waits while all nodes are noncritical; once a critical node appears, it
immediately chooses the highest-weight feasible edge incident to a critical
node. Because every exit is preceded by a warning, waiting while all nodes are
noncritical preserves market thickness without risking an unobserved
departure; the critical transition signals that further delay has become
costly. This implements the logic of patient policies
\citep{akbarpour2020thickness,kakimura2026dynamic}: preserve the pool while
waiting is safe and intervene when exit becomes imminent.

The GNN follows the same broad timing pattern and therefore often makes similar
choices, helping explain why it performs similarly to Patient Greedy and why
both substantially outperform Immediate Greedy, which removes nodes before
better future matching opportunities develop. However, we observe two state-dependent deviations from Patient Greedy that may
help explain the GNN's modest reward advantage. First, when a critical node is
present, the GNN sometimes chooses a slightly lower-weight edge incident to
the critical node to avoid leaving another node isolated and potentially
unmatchable. Second, it sometimes matches hard-to-match pairs before they
become critical. For example, type-O patients can receive only from type-O
donors, so their pairs are less likely to gain new feasible edges; once a
valuable edge is available, waiting may add discounting costs without
materially improving future opportunities.

\begin{figure}
    \centering
    \includegraphics[width=\columnwidth]{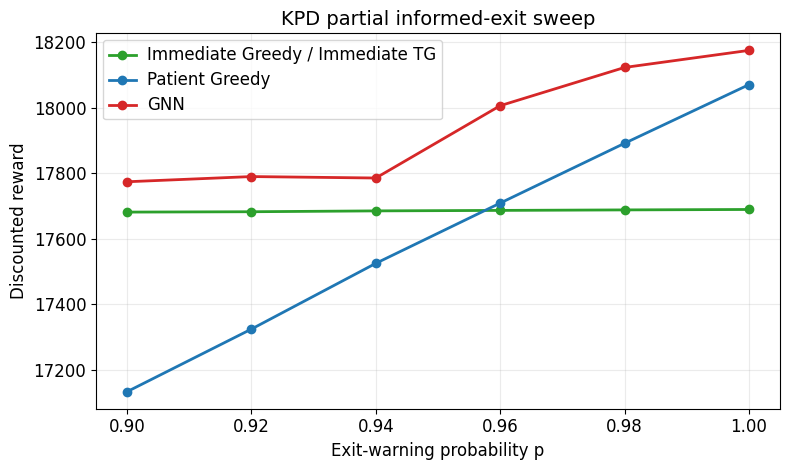}
    \caption{Discounted reward by exit-warning probability}
    \label{fig:kpd_intermediate}
\end{figure}

Finally, we examine whether the learned policy adapts flexibly to intermediate exit-information environments. When the clock of a
noncritical node rings, the node becomes critical with probability $p$ and
exits immediately with probability $1-p$. We refer to $p$ as the \emph{exit-warning probability}.
The case $p=0$ corresponds to uninformed exit, whereas
$p=1$ corresponds to informed exit. As $p$
increases, a patient policy should become more attractive.

Figure~\ref{fig:kpd_intermediate} reports the results. 
Patient Greedy performs poorly for small $p$ because many nodes exit
unmatched without exit warning, whereas the policy matches only
after a critical node arises.
By contrast, Immediate Greedy does not use the
critical flag and is insensitive to $p$. The optimal threshold is zero for
every $p$, so Immediate Threshold Greedy coincides with Immediate Greedy throughout.
In our simulations, Patient Greedy becomes comparable to Immediate Greedy
only around $p=0.96$. GNN outperforms the better of the two baselines.
When warnings are relatively rare, it behaves more like an immediate policy; when
warnings are reliable, it recovers the logic of patient matching. These
results suggest that GNN successfully learns a matching policy that adapts to the
degree of exit information available in the environment.

\section{Conclusion}

We proposed a graph neural reinforcement-learning framework for dynamic
matching on evolving weighted graphs. An event-time reduction shows that the
planner can choose a matching immediately after each exogenous event and then
wait for the next, expressing the problem through post-decision residual-graph
values and shifting learning from state-action to graph values. A GNN
approximates these values, while forward-greedy search handles combinatorial
action selection.

The learned policy captures state dependence absent from fixed rules. In the
binary-type benchmark, it balances current match value against market
thickness; in KPD, it adapts to exit-warning information and residual-pool
structure. It matches or outperforms Greedy and Patient Greedy. Although these
comparisons do not establish optimality, they show that residual-graph value
guidance can exploit economically meaningful state dependence.


\section*{Acknowledgments}

We are grateful to Itai Ashlagi, Michihiro Kandori, Yash Kanoria, and Fuhito Kojima for helpful comments. This work has been supported by JST PRESTO Grant Number JPMJPR2368 and JST ERATO Grant Number JRMJER2301, Japan. All remaining errors are our own.

\bibliography{ref}

@article{akbarpour2020thickness,
  title={Thickness and information in dynamic matching markets},
  author={Akbarpour, Mohammad and Li, Shengwu and Gharan, Shayan Oveis},
  journal={Journal of Political Economy},
  volume={128},
  number={3},
  pages={783--815},
  year={2020},
  publisher={The University of Chicago Press Chicago, IL}
}

@article{huang2024online,
  author  = {Huang, Zhiyi and Tang, Zhihao Gavin and Wajc, David},
  title   = {Online matching: A brief survey},
  journal = {ACM SIGecom Exchanges},
  volume  = {22},
  number  = {1},
  pages   = {135--158},
  year    = {2024},
  doi     = {10.1145/3699824.3699837}
}

@article{you2024approximate,
  author  = {You, Fan and Vossen, Thomas},
  title   = {An approximate dynamic programming approach to dynamic
             stochastic matching},
  journal = {INFORMS Journal on Computing},
  volume  = {36},
  number  = {4},
  pages   = {1006--1022},
  year    = {2024},
  doi     = {10.1287/ijoc.2021.0203}
}

@inproceedings{schaul2016prioritized,
  title={Prioritized experience replay},
  author={Schaul, Tom and Quan, John and Antonoglou, Ioannis and Silver, David},
  booktitle={International Conference on Learning Representations},
  year={2016},
  url={https://arxiv.org/abs/1511.05952}
}

@inproceedings{kingma2015adam,
  title={Adam: a method for stochastic optimization},
  author={Kingma, Diederik P. and Ba, Jimmy},
  booktitle={International Conference on Learning Representations},
  year={2015},
  url={https://arxiv.org/abs/1412.6980}
}

@inproceedings{dutting2019optimal,
  title={Optimal auctions through deep learning},
  author={D{\"u}tting, Paul and Feng, Zhe and Narasimhan, Harikrishna and Parkes, David and Ravindranath, Sai Srivatsa},
  booktitle={International Conference on Machine Learning},
  pages={1706--1715},
  year={2019},
  organization={PMLR}
}

@article{sutton1988learning,
  title={Learning to predict by the methods of temporal differences},
  author={Sutton, Richard S},
  journal={Machine Learning},
  volume={3},
  number={1},
  pages={9--44},
  year={1988},
  publisher={Springer}
}

@article{anderson2017efficient,
  title={Efficient dynamic barter exchange},
  author={Anderson, Ross and Ashlagi, Itai and Gamarnik, David and Kanoria, Yash},
  journal={Operations Research},
  volume={65},
  number={6},
  pages={1446--1459},
  year={2017},
  publisher={INFORMS}
}

@article{ashlagi2019matching,
  title={On matching and thickness in heterogeneous dynamic markets},
  author={Ashlagi, Itai and Burq, Maximilien and Jaillet, Patrick and Manshadi, Vahideh},
  journal={Operations Research},
  volume={67},
  number={4},
  pages={927--949},
  year={2019},
  publisher={INFORMS}
}

@unpublished{kakimura2026dynamic,
author = {Kakimura, Naonori and Zhu, Donghao},
title = {Dynamic bipartite matching markets with stochastic arrivals and departures},
note = {Mathematics of Operations Research, forthcoming},
year = {2026}
}

@article{roth2007efficient,
  author = {Roth, Alvin E. and S{\"o}nmez, Tayfun and {\"U}nver, M. Utku},
  title = {Efficient kidney exchange: coincidence of wants in markets with compatibility-based preferences},
  journal = {American Economic Review},
  year = {2007},
  volume = {97},
  number = {3},
  pages = {828--851},
  doi = {10.1257/aer.97.3.828},
  url = {https://www.aeaweb.org/articles?id=10.1257/aer.97.3.828}
}

@article{holscher2018matchrates,
  author = {Holscher, Courtenay M. and Jackson, Kyle and Chow, Eric K. H. and Thomas, Alvin G. and Haugen, Christine E. and DiBrito, Sandra R. and Purcell, Carlin and Ronin, Matthew and Waterman, Amy D. and Garonzik Wang, Jacqueline and Massie, Allan B. and Gentry, Sommer E. and Segev, Dorry L.},
  title = {Kidney exchange match rates in a large multicenter clearinghouse},
  journal = {American Journal of Transplantation},
  year = {2018},
  volume = {18},
  number = {6},
  pages = {1510--1517},
  doi = {10.1111/ajt.14689},
  url = {https://pubmed.ncbi.nlm.nih.gov/29437286/}
}

@article{leeser2012living,
  author = {Leeser, David B. and Aull, Meredith J. and Afaneh, Cheguevara
          and Dadhania, Darshana and Charlton, Marian
          and Walker, Jennifer K. and Hartono, Choli and Serur, David
          and Del Pizzo, Joseph J. and Kapur, Sandip},
  title = {Living donor kidney paired donation transplantation: experience as a founding member center of the National Kidney Registry},
  journal = {Clinical Transplantation},
  year = {2012},
  volume = {26},
  number = {3},
  pages = {E213--E222},
  doi = {10.1111/j.1399-0012.2012.01606.x},
  url = {https://pubmed.ncbi.nlm.nih.gov/22872872/}
}

@article{milner2016hla,
  author = {Milner, John and Melcher, Marc L. and Lee, Brian and Veale, Jeffrey and Ronin, Matthew and D'Alessandro, Tom and Hil, Garet and Fry, Phillip C. and Shannon, Patrick W.},
  title = {{HLA} matching trumps donor age: donor-recipient pairing characteristics that impact long-term success in living donor kidney transplantation in the era of paired kidney exchange},
  journal = {Transplantation Direct},
  year = {2016},
  volume = {2},
  number = {7},
  pages = {e85},
  doi = {10.1097/TXD.0000000000000597},
  url = {https://pubmed.ncbi.nlm.nih.gov/27830179/}
}

@article{lentine2023adrkidney,
  author = {Lentine, Krista L. and Smith, Jodi M. and Miller, Jonathan M. and Bradbrook, Keighly and Larkin, Lindsay and Weiss, Samantha and Handarova, Dzhuliyana K. and Temple, Kayla and Israni, Ajay K. and Snyder, Jon J.},
  title = {{OPTN/SRTR} 2021 annual data report: kidney},
  journal = {American Journal of Transplantation},
  year = {2023},
  volume = {23},
  number = {2 Suppl 1},
  pages = {S21--S120},
  doi = {10.1016/j.ajt.2023.02.004},
  url = {https://www.sciencedirect.com/science/article/pii/S1600613523002514}
}

@misc{unos_abo_chart,
  author = {{United Network for Organ Sharing}},
  title = {Organ transplant: donor and recipient {ABO} blood type compatibility chart},
  year = {2026},
  howpublished = {PDF},
  url = {https://unos.org/wp-content/uploads/ABO-Donor-Recipient-Compatibility-Chart-v6-1.pdf},
  note = {Accessed 2026-04-16}
}

@article{DBLP:journals/ior/AouadS22,
  author       = {Ali Aouad and
                  {\"{O}}mer Sarita{\c{c}}},
  title        = {Dynamic stochastic matching under limited time},
  journal      = {Operations Research},
  volume       = {70},
  number       = {4},
  pages        = {2349--2383},
  year         = {2022},
  url          = {https://doi.org/10.1287/opre.2022.2293},
  doi          = {10.1287/OPRE.2022.2293},
  bibsource    = {dblp computer science bibliography, https://dblp.org}
}

@inproceedings{DBLP:conf/acml/BegnardiBJ023,
  author       = {Luca Begnardi and
                  Hendrik Baier and
                  Willem van Jaarsveld and
                  Yingqian Zhang},
  editor       = {Berrin Yanikoglu and
                  Wray L. Buntine},
  title        = {Deep reinforcement learning for two-sided online bipartite matching in collaborative order picking},
  booktitle    = {Asian Conference on Machine Learning, {ACML} 2023, 11-14 November
                  2023, Istanbul, Turkey},
  series       = {Proceedings of Machine Learning Research},
  pages        = {121--136},
  publisher    = {{PMLR}},
  year         = {2023},
  url          = {https://proceedings.mlr.press/v222/begnardi24a.html},
  bibsource    = {dblp computer science bibliography, https://dblp.org}
}

@inproceedings{DBLP:conf/icml/HayderiSVW24,
  author       = {Alexandre Hayderi and
                  Amin Saberi and
                  Ellen Vitercik and
                  Anders Wikum},
  editor       = {Ruslan Salakhutdinov and
                  Zico Kolter and
                  Katherine A. Heller and
                  Adrian Weller and
                  Nuria Oliver and
                  Jonathan Scarlett and
                  Felix Berkenkamp},
  title        = {{MAGNOLIA:} matching algorithms via {GNNs} for online value-to-go approximation},
  booktitle    = {Forty-first International Conference on Machine Learning, {ICML} 2024,
                  Vienna, Austria, July 21-27, 2024},
  series       = {Proceedings of Machine Learning Research},
  pages        = {17759--17782},
  publisher    = {{PMLR} / OpenReview.net},
  year         = {2024},
  url          = {https://proceedings.mlr.press/v235/hayderi24a.html},
  bibsource    = {dblp computer science bibliography, https://dblp.org}
}

@article{DBLP:journals/tits/HuFL25,
  author       = {Yulong Hu and
                  Siyuan Feng and
                  Sen Li},
  title        = {{BMG-Q:} localized bipartite match graph attention {Q}-learning for ride-pooling order dispatch},
  journal      = {{IEEE} Transactions on Intelligent Transportation Systems},
  volume       = {26},
  number       = {10},
  pages        = {15407--15421},
  year         = {2025},
  url          = {https://doi.org/10.1109/TITS.2025.3595653},
  doi          = {10.1109/TITS.2025.3595653},
  bibsource    = {dblp computer science bibliography, https://dblp.org}
}

@article{Xing2025RidePoolingSTGNN,
  author  = {Xing, Xue and Peng, Yuqi and Wan, Le and Luo, Fahui},
  title   = {Optimization of two-passenger ride-pooling orders based on {ST-GNN} and path optimization},
  journal = {PLOS ONE},
  year    = {2025},
  volume  = {20},
  number  = {12},
  pages   = {e0337415},
  doi     = {10.1371/journal.pone.0337415},
  url     = {https://doi.org/10.1371/journal.pone.0337415}
}

@article{chen2025feature,
  author  = {Chen, Yilun and Kanoria, Yash and Kumar, Akshit
             and Zhang, Wenxin},
  title   = {Feature-based dynamic matching},
  journal = {Operations Research},
  volume  = {74},
  number  = {2},
  pages   = {788--803},
  year    = {2026},
  doi     = {10.1287/opre.2024.0730}
}

@unpublished{liu2026datadriven,
  author  = {Liu, Weiliang and Ward, Amy R. and Zhang, Xun},
  title   = {Data-driven matching for impatient and heterogeneous
             demand and supply},
  note= {Mathematics of Operations Research, forthcoming},
  year    = {2026}
}

@inproceedings{Liu2025linksage,
author = {Liu, Ping and Wei, Haichao and Hou, Xiaochen and Shen, Jianqiang and He, Shihai and Shen, Qianqi and Chen, Zhujun and Borisyuk, Fedor and Hewlett, Daniel and Wu, Liang and Veeraraghavan, Srikant and Tsun, Alex and Jiang, Chengming and Zhang, Wenjing},
title = {{LinkSAGE}: optimizing job matching using graph neural networks},
year = {2025},
isbn = {9798400712456},
publisher = {Association for Computing Machinery},
address = {New York, NY, USA},
url = {https://doi.org/10.1145/3690624.3709396},
doi = {10.1145/3690624.3709396},
booktitle = {Proceedings of the 31st ACM SIGKDD Conference on Knowledge Discovery and Data Mining V.1},
pages = {2448–2457},
numpages = {10},
location = {Toronto ON, Canada},
series = {KDD '25}
}

@inproceedings{zhou2026hybrid,
  author    = {Zhou, Ruiqi and Zhu, Donghao and Shen, Houcai},
  title     = {A learning-based hybrid decision framework for
               matching systems with user departure detection},
  booktitle = {Human Interface and the Management of Information},
  series    = {Lecture Notes in Computer Science},
  volume    = {16706},
  pages     = {633--651},
  publisher = {Springer},
  year      = {2026},
  doi       = {10.1007/978-3-032-29178-3_42}
}

@article{eom2025batching,
  author  = {Eom, Myungeun and Toriello, Alejandro},
  title   = {Batching and greedy policies: How good are they
             in dynamic matching?},
  journal = {Manufacturing \& Service Operations Management},
  volume  = {28},
  number  = {2},
  pages   = {479--495},
  year    = {2026},
  doi     = {10.1287/msom.2024.1074}
}

@article{10.1287/mnsc.2022.00096,
author = {Castillo, Juan Camilo and Knoepfle, Dan and Weyl, E. Glen},
title = {Matching and pricing in ride hailing: wild goose chases and how to solve them},
year = {2025},
issue_date = {May 2025},
publisher = {INFORMS},
address = {Linthicum, MD, USA},
volume = {71},
number = {5},
issn = {0025-1909},
url = {https://doi.org/10.1287/mnsc.2022.00096},
doi = {10.1287/mnsc.2022.00096},
journal = {Management Science},
month = may,
pages = {4377–4395},
numpages = {19}
}

@article{DBLP:journals/nature/MnihKSRVBGRFOPB15,
  author       = {Volodymyr Mnih and
                  Koray Kavukcuoglu and
                  David Silver and
                  Andrei A. Rusu and
                  Joel Veness and
                  Marc G. Bellemare and
                  Alex Graves and
                  Martin A. Riedmiller and
                  Andreas Fidjeland and
                  Georg Ostrovski and
                  Stig Petersen and
                  Charles Beattie and
                  Amir Sadik and
                  Ioannis Antonoglou and
                  Helen King and
                  Dharshan Kumaran and
                  Daan Wierstra and
                  Shane Legg and
                  Demis Hassabis},
  title        = {Human-level control through deep reinforcement learning},
  journal      = {Nat.},
  volume       = {518},
  number       = {7540},
  pages        = {529--533},
  year         = {2015},
  url          = {https://doi.org/10.1038/nature14236},
  doi          = {10.1038/NATURE14236},
  bibsource    = {dblp computer science bibliography, https://dblp.org}
}

@misc{arnosti2025greedydynamicmatching,
      title={Greedy dynamic matching}, 
      author={Nick Arnosti and Felipe Simon},
      year={2025},
      eprint={2507.04551},
      archivePrefix={arXiv},
      primaryClass={cs.DS},
      url={https://arxiv.org/abs/2507.04551}, 
}


\appendix

\section{Proofs}

\subsection{Proof of Lemma~\ref{lem:no_delay}}

\begin{proof}
Let $\Delta_R\sim \Exp(\Lambda(R))$ be the waiting time to the next
exogenous event, let $\xi_R$ be its mark, and write
\[
Y_R\coloneqq r^{\mathrm{ex}}(R,\xi_R)+U(T(R,\xi_R)).
\]
Consider a plan that schedules matching $M\in\M(R)$ at time $\Sigma$ if
no exogenous event occurs first. With the convention
$e^{-a\cdot \infty}=0$ for $a>0$, its continuation value is
\begin{align}
&\E\!\left[
e^{-\delta \Delta_R}Y_R\mathbf 1\{\Delta_R\le \Sigma\}
\ \middle|\ R
\right]
\nonumber\\
&\quad{}
+
\E\!\left[
\begin{gathered}
e^{-\delta \Sigma}
\bigl(W(R,M)+U(R\res M)\bigr)\\
\times \mathbf 1\{\Sigma<\Delta_R\}
\end{gathered}
\ \middle|\ R
\right].
\end{align}

Condition on a realization $(\Sigma,M)=(s,m)$. Since
$(\Sigma,M)$ is independent of $(\Delta_R,\xi_R)$ conditional on $R$,
and since $\Delta_R\sim \Exp(\Lambda(R))$,
\begin{align}
&\E\!\left[
e^{-\delta \Delta_R}Y_R\,\mathbf 1\{\Delta_R\le s\}
\ \middle|\ R,\Sigma=s,M=m
\right]
\nonumber\\
&=
\int_0^s e^{-\delta t}\Lambda(R)e^{-\Lambda(R)t}\,\dd t
\cdot \E[Y_R\mid R]
\nonumber\\
&=
\left(
\int_0^s e^{-(\delta+\Lambda(R))t}\Lambda(R)\,\dd t
\right)\E[Y_R\mid R].
\label{eq:delay_part1}
\end{align}
By \eqref{eq:bellman_post_raw},
\begin{equation}
V(R)
=
\int_0^\infty e^{-(\delta+\Lambda(R))t}\Lambda(R)\,\dd t
\cdot \E[Y_R\mid R].
\end{equation}
Therefore
\begin{align}
&\E\!\left[
e^{-\delta \Delta_R}Y_R\mathbf 1\{\Delta_R\le s\}
\ \middle|\ R,\Sigma=s,M=m
\right]
\nonumber\\
&\quad=
\bigl(1-e^{-(\delta+\Lambda(R))s}\bigr)V(R).
\label{eq:delay_part1_simplified}
\end{align}
Similarly,
let $A_m\coloneqq W(R,m)+U(R\res m)$. Then
\begin{align}
&\E\!\left[
e^{-\delta s}
A_m
\mathbf 1\{s<\Delta_R\}
\ \middle|\ R,\Sigma=s,M=m
\right]
\label{eq:delay_part2}
\\
&\qquad=
e^{-\delta s}\Prob(\Delta_R>s\mid R)
A_m
\nonumber\\
&\qquad=
e^{-(\delta+\Lambda(R))s} A_m.
\nonumber
\end{align}
Combining \eqref{eq:delay_part1_simplified} and \eqref{eq:delay_part2}
shows that, conditional on $(\Sigma,M)$, the plan has value
\begin{equation}
V(R)
+
e^{-(\delta+\Lambda(R))\Sigma}
\bigl(
W(R,M)+U(R\res M)-V(R)
\bigr).
\end{equation}
This is a convex combination of $V(R)$ and $W(R,M)+U(R\res M)$, and is
therefore bounded above by the larger of the two. Taking expectations over
$(\Sigma,M)$ gives \eqref{eq:no_delay_bound}.
\end{proof}

\subsection{Proof of Theorem~\ref{thm:event_time}}

\begin{proof}
We prove \eqref{eq:bellman_pre} by induction on the number of nodes
$n\coloneqq|N(G)|$.

\medskip
\noindent
\paragraph{Base case}
If $n\le 1$, then no nonempty matching is feasible, so
$\M(G)=\{\emptyset\}$. Hence the planner cannot obtain any immediate
matching payoff and must wait until the next exogenous event. Therefore
\begin{equation}
U(G)=V(G)=W(G,\emptyset)+V(G\res \emptyset),
\end{equation}
which is exactly \eqref{eq:bellman_pre}.

\medskip

\paragraph{Induction step}
Fix $n\ge 2$, and assume that \eqref{eq:bellman_pre} holds for every
graph with fewer than $n$ nodes. Let $G$ be a graph with $|N(G)|=n$,
and let $\pi$ be any admissible policy.

Condition on a realization of the private randomization used by $\pi$.
After conditioning, the policy becomes deterministic between exogenous
events. Along the counterfactual path on which no exogenous event occurs
after the current date, $\pi$ therefore induces a deterministic sequence
of matching actions scheduled at deterministic times. Since every nonempty
matching removes at least two nodes, there can be only finitely many
nonempty matching actions before the next exogenous event on this
counterfactual path.

Let $m_0\in \M(G)$ be the union of all matchings that $\pi$ executes at
the current time $0$. Because matched nodes are removed immediately, these
matchings are pairwise disjoint, so their union is again a matching in $G$. Let $R_0\coloneqq G\res m_0$. If $\pi$ takes no further matching action before the next exogenous event,
then its continuation value from the current date is exactly
\begin{align}
&W(G,m_0)+V(R_0)\\
& \le \max_{m\in \M(G)}
\left\{
W(G,m)+V(G\res m)
\right\}.
\end{align}

Now suppose instead that $\pi$ schedules a further matching action before
the next exogenous event. Let $\sigma>0$ be the first strictly positive
scheduled intervention time on the no-event path, and let
$m_1\in \M(R_0)$ be the matching that $\pi$ plans to execute at time
$\sigma$. Since nothing happens between time $0$ and time $\sigma$ on
the no-event path, $m_1$ is feasible in $R_0$. Applying
Lemma~\ref{lem:no_delay} to the graph $R_0$ with the deterministic choice
$(\Sigma,M)=(\sigma,m_1)$ yields
\begin{align}
J^\pi(G)
&\le
\label{eq:event_time_induction_step}
\max\left\{
\begin{array}{@{}l@{}}
W(G,m_0)+V(R_0),\\
W(G,m_0)+W(R_0,m_1)\\
\quad{}+U(R_0\res m_1)
\end{array}
\right\}
\\
&=
\max\left\{
\begin{array}{@{}l@{}}
W(G,m_0)+V(R_0),\\
W(G,m_0\cup m_1)\\
\quad{}+U(G\res (m_0\cup m_1))
\end{array}
\right\}.
\nonumber
\end{align}
The first term is already bounded above by the right-hand side of
\eqref{eq:bellman_pre}. Since $m_1\neq \emptyset$, the graph $G_1\coloneqq G\res (m_0\cup m_1)$ has strictly fewer than $n$ nodes. Hence the induction hypothesis implies
\begin{equation}
U(G_1)
=
\max_{\hat m\in \M(G_1)}
\left\{
W(G_1,\hat m)+V(G_1\res \hat m)
\right\}.
\end{equation}
Substituting this into the second term of
\eqref{eq:event_time_induction_step} gives
\begin{align}
\label{eq:event_time_induction_finish}
&W(G,m_0\cup m_1)+U(G_1)
\\
&=
W(G,m_0\cup m_1)
+
\max_{\hat m\in \M(G_1)}
\left\{
W(G_1,\hat m)+V(G_1\res \hat m)
\right\}
\\
&=
\max_{\hat m\in \M(G_1)}
\left\{
\begin{aligned}
&W(G,m_0\cup m_1)+W(G_1,\hat m)\\
&\quad{}+V(G_1\res \hat m)
\end{aligned}
\right\}
\\
&\le
\max_{m\in \M(G)}
\left\{
W(G,m)+V(G\res m)
\right\}.
\end{align}
Combining the preceding bounds shows that, conditional on the realized
private randomization of $\pi$,
\begin{equation}
J^\pi(G)
\le
\max_{m\in \M(G)}
\left\{
W(G,m)+V(G\res m)
\right\}.
\end{equation}
Since the right-hand side does not depend on the realization of the private
randomization, the same inequality holds after integrating over that
randomization. Because $\pi$ was arbitrary, we conclude that
\begin{equation}
U(G)
\le
\max_{m\in \M(G)}
\left\{
W(G,m)+V(G\res m)
\right\}.
\end{equation}

For the reverse inequality, fix any $m\in \M(G)$. The policy that
executes $m$ immediately at the current event time and then waits until
the next exogenous event is admissible and yields value
\begin{equation}
W(G,m)+V(G\res m).
\end{equation}
Taking the maximum over $m\in \M(G)$ gives
\begin{equation}
U(G)
\ge
\max_{m\in \M(G)}
\left\{
W(G,m)+V(G\res m)
\right\}.
\end{equation}
This proves \eqref{eq:bellman_pre}.

Finally, $\M(G)$ is finite for every finite graph $G$, so the maximum in
\eqref{eq:bellman_pre} is attained. Choosing, for example, the
lexicographically smallest maximizer after each realized exogenous event
defines an optimal event-time policy. By construction, that policy acts only
immediately after exogenous events and takes exactly one matching decision at
each such event time.
\end{proof}

\subsection{Proof of Theorem~\ref{thm:q_to_v}}

\begin{proof}
If the planner chooses $e\in E(G)$, then it receives payoff $w_e$
immediately and moves deterministically to $G\res \{e\}$. Since this
transition is instantaneous, no discount is incurred between $G$ and
$G\res \{e\}$. Therefore
\begin{equation}
Q(G,e)=w_e+U(G\res \{e\}),
\qquad e\in E(G),
\end{equation}
which proves \eqref{eq:q_edge}.

If the planner chooses $\perp$, then no further matching action is taken
in the current event epoch and the process enters the post-decision waiting
stage from the same graph $G$. By definition,
\begin{equation}
Q(G,\perp)=V(G),
\end{equation}
which proves \eqref{eq:q_stop}.

Now consider any feasible finite sequence of edge actions $e_1,e_2,\dots,e_K,\perp$ starting from $G$. Because the endpoints of $e_k$ are deleted as soon as $e_k$ is chosen, no later edge can share an endpoint with any earlier edge. Hence the edges $e_1,\dots,e_K$ are pairwise disjoint, and
\begin{equation}
m\coloneqq\{e_1,\dots,e_K\}\in \M(G).
\end{equation}
Since all edge choices occur instantaneously within the same event epoch,
the value of the sequence is
\begin{equation}
\sum_{k=1}^K w_{e_k}+V(G\res m)
=
W(G,m)+V(G\res m).
\end{equation}
This value depends only on the underlying matching $m$, not on the order
in which its edges are chosen.

Conversely, fix any matching $m\in \M(G)$, and write it as
$m=\{e_1,\dots,e_K\}$ in an arbitrary order. Since the edges in $m$ are
pairwise disjoint, the sequence $e_1,e_2,\dots,e_K,\perp$ is feasible and yields value $W(G,m)+V(G\res m)$. 
Therefore maximizing over feasible edge-action sequences is equivalent to
maximizing over feasible matchings. By
Theorem~\ref{thm:event_time}, $U(G) = \max_{m\in \M(G)} \left\{W(G,m)+V(G\res m)\right\}$,
and hence
\begin{equation}
U(G)=
\max\Bigl(
\{Q(G,\perp)\}\cup\{Q(G,e):e\in E(G)\}
\Bigr).
\end{equation}
This proves \eqref{eq:q_to_v_reduction}.

Finally, the formulas above show that once $V$ is known, one recovers
$U$ from \eqref{eq:bellman_pre} and then recovers $Q(G,e)$ and
$Q(G,\perp)$ from \eqref{eq:q_edge}--\eqref{eq:q_stop}. Conversely, given
the optimal $Q$-function, one recovers $V(G)$ from $Q(G,\perp)$ and
$U(G)$ from the maximum over $Q(G,\perp)$ and the values $Q(G,e)$ for
$e\in E(G)$.
Thus the optimal edge-wise $Q$-function is uniquely determined by the
single post-decision value function $V$.
\end{proof}

\section{Kidney Exchange Simulation Details}
\label{app:kpd_details}

This appendix gives the full parameterization of the KPD environments used
in the experiments. Each node is an incompatible patient--donor pair. Let
$\mathcal B\coloneqq\{O,A,B,AB\}$ be the set of ABO blood types. A patient
type is $P=(b^P,s)\in \mathcal B\times\{L,H\}$, where $s=L$ denotes a
lower-sensitized patient and $s=H$ denotes a highly sensitized patient. A
donor type is $D=(b^D,a,g)\in\mathcal B\times\{30,45,60\}\times\{F,M\}$,
where the donor-age support is represented by coarse age-band midpoints.
The node type is $X=(P,D)$.

The patient ABO probabilities are $(0.48,0.34,0.14,0.04)$ for
$(O,A,B,AB)$. These are the classical kidney-exchange simulation
frequencies in \citet[Table~1]{roth2007efficient}, rounded to two decimal
places. The donor ABO probabilities are $(0.45,0.40,0.11,0.04)$ for
$(O,A,B,AB)$; we use a similar donor-side margin so that blood types $O$
and $A$ dominate while keeping the model low-dimensional. The sensitization
probabilities are $\sigma_L=0.70$ and $\sigma_H=0.30$. This two-class split
keeps the benchmark coarse while preserving the empirically important fact
that highly sensitized patients are difficult to match; this difficulty is
also visible in waiting-time evidence from KPD pools \citep{holscher2018matchrates}.
Donor age probabilities are $(0.30,0.45,0.25)$ for age bands
$(30,45,60)$, and donor sex probabilities are $\gamma_F=0.55$ and
$\gamma_M=0.45$. These values are deliberately smoothed but are chosen to
reflect the broad empirical regularities that middle-aged donors are common
and women are overrepresented among living kidney donors
\citep{lentine2023adrkidney}.

New pairs arrive according to a Poisson process with rate $\lambda=10.0$.
This rate is not meant as a direct registry estimate. It is chosen to make
residual graphs contain enough simultaneous feasible edges for nontrivial
policy choices in the simulated benchmark. The same design goal also guides
the crossmatch and exit-rate choices described below.

The arrival distribution combines the above marginals with a KPD-entry
filter. The filter is $\kappa(b^P,b^D,s)=1$ if donor ABO type $b^D$ is
incompatible with patient ABO type $b^P$, $\kappa(b^P,b^D,s)=\eta_H$ if
the pair is ABO-compatible but $s=H$, and $\kappa(b^P,b^D,s)=0$ if the
pair is ABO-compatible and $s=L$, with $\eta_H=0.20$. This reflects the
clinical logic that KPD pools are primarily populated by blood-type or
crossmatch incompatible pairs \citep[p.~828]{roth2007efficient}. The
positive mass on ABO-compatible, highly sensitized pairs captures the fact
that immunologic incompatibility can also bring pairs into exchange; in a
founding-center NKR study, \citet{leeser2012living} report that
incompatibility was attributable to blood type in $54.4\%$ of enrolled
pairs and to donor-specific sensitization in $43.2\%$. Formally,
\begin{equation}
\rho\bigl(((b^P,s),(b^D,a,g))\bigr)=
\frac{
\pi^P_{b^P}\sigma_s\pi^D_{b^D}\alpha_a\gamma_g
}{Z}\times
\kappa(b^P,b^D,s),
\label{eq:kpd_arrival_distribution}
\end{equation}
where $Z$ normalizes $\rho$ to sum to one.

For two nodes $X_i=((b_i^P,s_i),(b_i^D,a_i,g_i))$ and
$X_j=((b_j^P,s_j),(b_j^D,a_j,g_j))$, a two-way exchange induces two
directed living-donor transplants. A directed transplant first requires ABO
compatibility; we use the standard compatibility rule that type $O$ donors
can donate to any ABO type, whereas type $O$ recipients can receive only
from type $O$ donors \citep{unos_abo_chart}. Conditional on ABO
compatibility, the transplant passes a virtual crossmatch with probability
$\phi_L=0.90$ for lower-sensitized patients and $\phi_H=0.50$ for highly
sensitized patients. These crossmatch probabilities are benchmark parameters. They preserve the ordering that highly sensitized
patients are harder to match, while keeping the simulated pool dense enough
for policy learning rather than almost always edgeless. The ordering is
consistent with the positive-crossmatch patterns reported in
\citet[p.~839]{roth2007efficient}.

Conditional on directed feasibility, the transplant receives an
edge-specific HLA match score $H\in\{0,25,50,75\}$ with probabilities
$(0.20,0.30,0.30,0.20)$. This coarse support covers the clinically
meaningful range emphasized by \citet{milner2016hla}, who compare outcomes
at $0$ and $75$ HLA match points and estimate how HLA quality affects
expected kidney life years. The directed transplant value uses their linear
KLY specification, omitting the biological-sibling term because KPD
exchanges are typically between unrelated donor--recipient pairs:
\begin{align}
u_{\mathrm{KLY}}(p,d;H)
&\coloneqq
23.465
+0.039H
-0.050a
\nonumber\\
&\quad{}
+0.303 \times \mathbf 1\{g=M\}.
\label{eq:kly_directed_utility}
\end{align}
The coefficients are the intercept, HLA-point effect, donor-age effect, and
donor-male effect reported by \citet{milner2016hla}. An undirected edge
between $i$ and $j$ is feasible if both directed transplants are feasible;
if so, its weight is the sum of the two directed KLY values:
\begin{align}
w_{ij}
&\coloneqq
u_{\mathrm{KLY}}\bigl((b_i^P,s_i),(b_j^D,a_j,g_j);H_{ij}\bigr)
\nonumber\\
&\quad{}
+u_{\mathrm{KLY}}\bigl((b_j^P,s_j),(b_i^D,a_i,g_i);H_{ji}\bigr).
\label{eq:kpd_edge_weight}
\end{align}
If either directed transplant is infeasible, the undirected edge is absent.

In the baseline KPD environment, unmatched nodes exit exogenously with
$\mu(x)=1.0$, $q(\exitmark\mid x)=1$, and $c(x)=0$ for all $x\in\X$.
The exit penalty remains zero to keep rewards focused on realized
transplants. Together with $\lambda=10.0$ and the crossmatch probabilities
above, this exit rate keeps the pool dynamic while still leaving enough
edges for meaningful matching choices.

The KPD input choice described in the main text affects only the value
network. The weights in \eqref{eq:kpd_edge_weight} remain in $W(G,m)$ and in
the forward-greedy score
$w_e+V_\theta(G\res\{e\})$. Donor age and sex are endpoint features, so
their KLY contributions are available to the network; $H_{ij}$ and $H_{ji}$
are the only edge-specific payoff components. In the architecture ablation,
supplying $w_{ij}$ as an edge feature produced essentially unchanged
evaluation reward but less stable training.

The scale of the KLY calibration is consistent with this result. A directed
transplant is worth at least
$23.465-0.050\times 60=20.465$ KLY, so every feasible two-way exchange is
worth at least $40.93$ KLY. By comparison, the combined edge-specific HLA
contribution ranges from $0$ to
$2\times0.039\times75=5.85$ KLY, while an exogenous exit yields zero reward.
Thus, in this benchmark, the main variation in continuation value is driven
by whether future feasible exchanges remain available and by exit risk,
rather than by fine HLA-based differences among feasible edges.

The informed-exit environment augments each node with an observable
critical flag $h\in\{0,1\}$. All arrivals enter with $h=0$. If $h=0$, the
node's clock rings at rate $\mu_0=1.0$ and deterministically changes the
flag to $h=1$ without redrawing or deleting incident edges. If $h=1$, the
node's clock rings at rate $\mu_1=100.0$ and triggers exogenous exit with
zero penalty. Edge sampling for new arrivals ignores the critical flag, so
critical and noncritical nodes use the same compatibility and value rules.
This makes the flag a pure survival annotation: the planner observes a
warning that a node is about to leave, but the node's matching opportunities
do not change at the flag transition. Finally, the continuous-time discount
rate is $\delta=0.002$, matching the value used in the simulation
configuration.

\end{document}